\documentclass{article}

\usepackage[final]{nips_2017}

\usepackage{times}
\usepackage{graphicx} % more modern
\usepackage{algorithm}
\usepackage{algorithmic}

\usepackage{hyperref}

\usepackage{fancyhdr}
\usepackage{color}
\usepackage{algorithm}
\usepackage{algorithmic}
\usepackage{natbib}
\usepackage{eso-pic} % used by \AddToShipoutPicture 
\usepackage{forloop}

\renewcommand{\cite}[1]{\citep{#1}}

\setcitestyle{authoryear,round,citesep={;},aysep={,},yysep={;}}

\ifdefined\nohyperref\else\ifdefined\hypersetup
  \definecolor{mydarkblue}{rgb}{0,0.08,0.45}
  \hypersetup{ %
    pdftitle={},
    pdfauthor={},
    pdfkeywords={},
    pdfborder=0 0 0,
    pdfpagemode=UseNone,
    colorlinks=true,
    linkcolor=mydarkblue,
    citecolor=mydarkblue,
    filecolor=mydarkblue,
    urlcolor=mydarkblue,
    pdfview=FitH}

  \ifdefined\isaccepted \else
    \hypersetup{pdfauthor={Anonymous Submission}}
  \fi
\fi\fi

\usepackage{caption}
\usepackage{subcaption}
\usepackage{amssymb}
\usepackage{amsthm}
\usepackage{amsmath}
\usepackage{xcolor}
\usepackage{soul}
\usepackage{framed}

\usepackage[utf8]{inputenc} % allow utf-8 input
\usepackage[T1]{fontenc}    % use 8-bit T1 fonts
\usepackage{url}            % simple URL typesetting
\usepackage{booktabs}       % professional-quality tables
\usepackage{amsfonts}       % blackboard math symbols
\usepackage{nicefrac}       % compact symbols for 1/2, etc.
\usepackage{microtype}      % microtypography

\colorlet{shadecolor}{yellow}
\graphicspath{{./figures/}}

\newtheorem{theorem}{Theorem}[section]
\newtheorem{lemma}[theorem]{Lemma}

\theoremstyle{definition}

\title{Effects of the optimisation of the margin distribution on generalisation in deep architectures}

\author{Lech ~Szymanski\\
    University of Otago\\
    \texttt{lechszym@cs.otago.ac.nz}\\
    \And
    Brendan ~McCane\\   
    University of Otago\\
    \texttt{mccane@cs.otago.ac.nz}\\
    \And
    Wei ~Gao\\
    Nanjing University\\
    \texttt{gaow@lamda.nju.edu.cn}
    \And
    Zhi-Hua ~Zhou \\
    Nanjing University\\
    \texttt{zhouzh@lamda.nju.edu.cn}
}
\begin{document} 

\maketitle

\begin{abstract} 
Despite being so vital to success of Support Vector Machines, the principle of separating margin maximisation is not used in deep learning.  We show that minimisation of margin variance and not maximisation of the margin is more suitable for improving generalisation in deep architectures.  We propose the Halfway loss function that minimises the Normalised Margin Variance (NMV) at the output of a deep learning models and evaluate its performance against the Softmax Cross-Entropy loss on the MNIST, smallNORB and CIFAR-10 datasets.
\end{abstract} 

\section{Introduction}

Support Vector Machines (SVM) guarantee best generalisation in a classification task for a chosen feature extraction function \cite{Cortes.etal:1995, Vapnik:1995}.  While the question of the choice of appropriate feature function (or its parameters) still remains, the training is guaranteed to give the optimal answer for the choice made.  This assurance of generalisation comes from the principle of maximising the margin of separation.

Boosting methods build a feature space during the training process from an ensemble of weak classifiers \cite{Schapire:1990}.  It has been shown that their resistance to overfitting is due to the effect these methods have on the distribution of points around the margin \cite{Shapire.etal:1998, Reyzin.etal:2006, Wang.etal:2011}.  \citet{Gao.etal:2013}  theoretically showed that AdaBoost is resistant to overfitting because it implicitly optimises the classification margin distribution by maximising average margin and minimising margin variance simultaneously.  In particular, they emphasised that the minimisation of margin variance is very important, which was ignored by most previous studies on learning algorithm design. \citet{Zhang.etal:2013} proposed the LDM which maximises average margin and minimises margin variance simultaneously, and achieved consistently better performance than SVMs; later, \citet{Zhang.etal:2016} proposed Optimal Margin Machine (ODM) which demonstrates even better performance.  

In this paper we take up the idea of margin distribution and apply it to deep learning.  We theorise that in deep architectures with traditional backpropagation training, maximising the margin of separation is not likely to positively affect generalisation.  However, we demonstrate that Halfway loss, which aims to minimise the normalised margin variance (NMV), does lead to improved generalisation in terms of outperforming the Softmax Cross-Entropy loss on the MNIST, smallNORB and CIFAR10 datasets.  

\section{Previous work}

A number of different approaches have made an effective use of the principle of margin maximisation in artificial neural networks.  \citet{Jayadeva.etal:2002} combined it with a decision tree-based training and \citet{Nishikawa.etal:2002} incorporated it into the CARVE algorithm \cite{Young.etal:1998}.  Both of these methods are based on a boosting-like training scheme, where the feature space is built one neuron (hypothesis) at a time, focusing on the remaining, incorrectly classified subset of the training data.  Although in spirit these methods pertain to neural networks, the performance leverage they gain thanks to maximising the margin has probably more to do with the boosting aspects of the feature building rather than the deep nature of the neural network used.

The meticulously named Maximum Margin Gradient Descent with adaptive learning rate (MMGDX) algorithm proposed by \cite{Ludwig.etal2010}'s works in a more traditional, fixed connectionist architecture trained with the backpropagation algorithm.  The success of that algorithm most likely lies in the fact that the proposed Means Squared Error (MSE)-like loss, not unlike the Halfway loss introduced in this paper, might in fact be also minimising the distribution of margin variance.  It also should be noted that MMGDX was tested only on single-hidden layer neural network with sigmoid activation function, and the superior performance only showcased on single-class problems.   

The above methods train neural networks by maximising the geometric margin.  Hence, in some sense, they consider the mean margin, but ignore the influence of margin variance (with the exception of MMGDX, which unintentionally might be reducing the variance).  The Halfway loss minimises the margin variance and is not limited to the sigmoid activation function.  We can test it on fully connected, as well as convolutional, neural networks with Rectifier Linear Unit (ReLU) activation function \cite{Hahnloser.etal:2003, Glorot.etal:2011} and compare its performance to Softmax Central-Entropy loss on multi-class image recognition datasets.  

%Margin distribution previous work: MDO \cite{Garg.etal:2003},  MAMC \cite{Pelckmans.etal:2008}, KM-OMD \cite{Aiolli.etal:2008}  - probably not needed, because we're building on \cite{Gao.etal:2013,Zhang.etal:20016}, which supersede this work.

\section{Margin}

Let's denote by $\mathcal{X}\in\mathbb{R}^d$ the instance space and by $\mathcal{Y}\in\{+1,-1\}$ the label set governed by some distribution $\mathcal{D}$ over $\mathcal{X}\times\mathcal{Y}$.  Let's assume that we have a set $S=\{(\mathbf{x}_1,y_1),(\mathbf{x}_2,y_2),...,(\mathbf{x}_m,y_m)\}$ of $m$ points drawn identically and independently from $D$.  Given some feature extraction function $\phi(\mathbf{x})\in\mathbb{R}^k$ and a linear classifier parametrised by bias $b$ and a unit vector $\mathbf{w}\in\mathbb{R}^k$, so that $|\mathbf{w}|=1$,  we can define the margin of instance $(\mathbf{x}_i,y_i)$, which is really a distance of the point from the classification boundary, as:
\begin{equation}\gamma_i=y_i\left [\mathbf{w}^T\phi(\mathbf{x}_i)+b\right ]\mbox{.}\label{eqn_margin}
\end{equation}
Given classification error as 
\begin{equation*}
\epsilon(\gamma_i)=\begin{cases}0	& \gamma_i>0 \\
					   1    & \mbox{otherwise, }	
\end{cases}
\end{equation*}
the goal of binary classification is to search for $\phi(\mathbf{x})$, the direction of unit vector $\mathbf{w}$, and value of $b$ such that the expectation of $\epsilon (\gamma$) over  distribution $\mathcal{D}$, is minimised.  Since typically $\mathcal{D}$ is not known, the best way to estimate the expectation is by computing the average of $\epsilon(\gamma_i)$ over the points from the training sample $S$ and then try to maximise it. \citet{Breiman:1999} believed that AdaBoost also tried to maximise the minimum margin. Later, \citet{Reyzin.etal:2006} claimed that AdaBoost emphasises the average margin or median margin. The average margin is also called "margin mean", defined as:
\begin{equation}
\mu=\frac{1}{m}\sum_{i=1}^m\gamma_i\mbox{.}
\end{equation}

\subsection{Margin across different feature spaces}
%In essence, the SVM training gives the best solution to classification for a chosen kernel function, its parameters, and selected degree of penalty for misclassification.  The task for the user is to make appropriate choices about the kernel transformation.

In SVMs the feature extracting function $\phi(\mathbf{x})$ is user-definable, but fixed during optimisation.  The learning process for a given choice of $\phi(\mathbf{x})$ is the search for $\mathbf{w}$ and $b$ that maximises the geometric margin while providing correct classification to the degree dictated by the choice of the soft margin parameter.  A fixed feature space and margin maximisation guarantees an upper bound on probability of misclassification \cite{Cristianini.etal:2000}.  The challenge with SVMs is to determine the best $\phi(\mathbf{x})$ by selecting the right kernel function and its parameters as well as an appropriate soft margin penalty factor for misclassification.

\begin{figure}[t!]
\vskip 0.2in
\centering
\includegraphics[width=0.6\columnwidth]{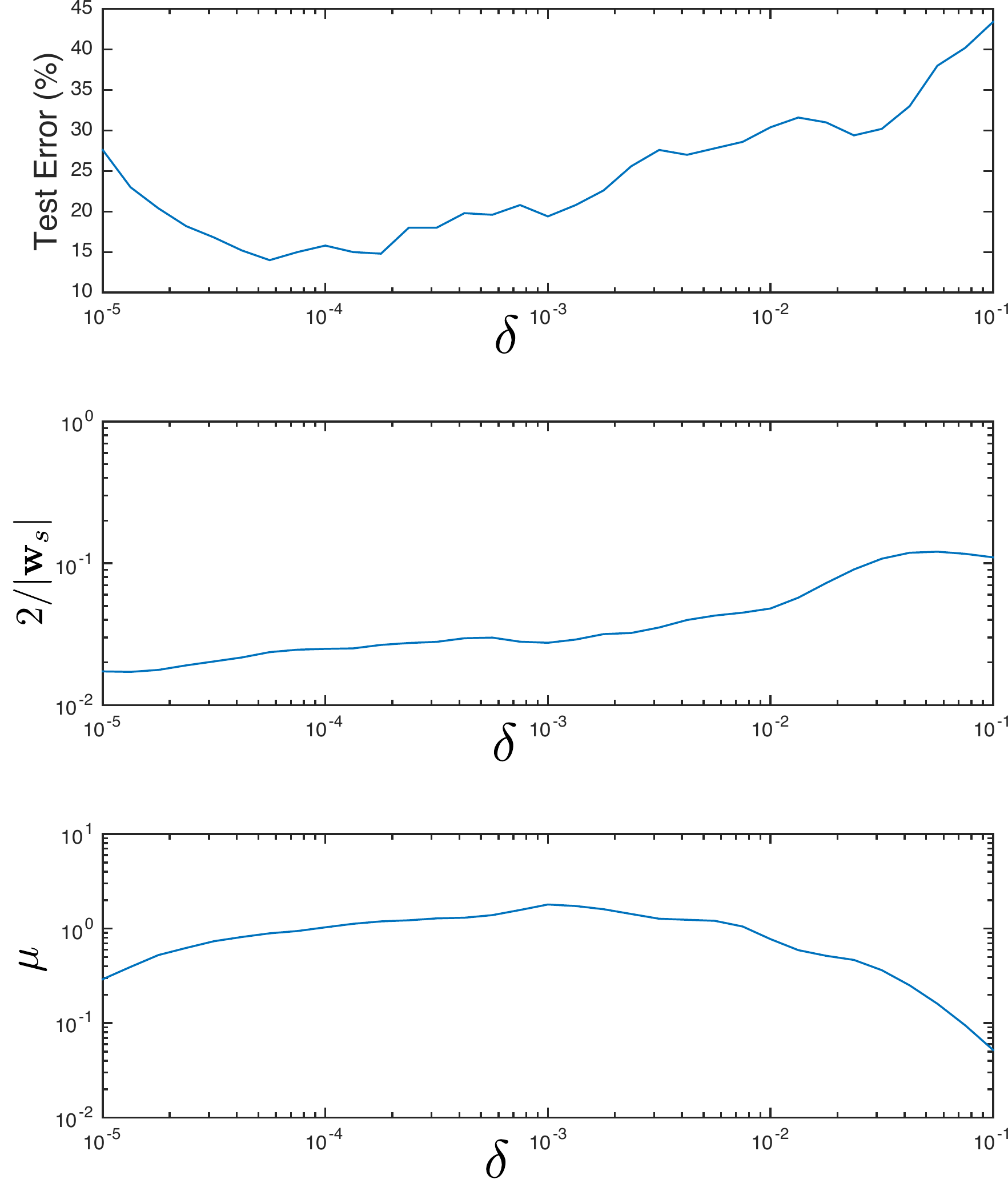}
\caption{Test error and the corresponding values for the maximum geometric margin and the mean margin found by SVM training on a two-class subset from the smallNORB dataset using feature extraction function defined in Equation \ref{eqn_rbfphi} over different values of parameter $\delta$.}\label{fig_margin}
\end{figure} 

Given our intention to apply margin theory to deep learning, we are prompted to investigate the behaviour of the margin in an SVM while varying $\phi(\mathbf{x})$.  \citet{Lanckriet.etal:2004} demonstrated that with some constraints and restrictions on $\phi(\mathbf{x})$, maximisation of the margin still provides an upper bound on probability of misclassification.  However, in deep learning $\phi(\mathbf{x})$ can be a universal function approximator, hence we are interested in margin behaviour in general.  Let us conduct a simple experiment.  

Let's define the following feature extraction function
\begin{equation}\label{eqn_rbfphi}
\phi(\mathbf{x})=\begin{bmatrix}e^{-\delta|\mathbf{x}-\mathbf{x}_1|} & e^{-\delta|\mathbf{x}-\mathbf{x}_2|} \hdots e^{-\delta|\mathbf{x}-\mathbf{x}_m|}\end{bmatrix}^T\mbox{,}
\end{equation}
where $\mathbf{x}_i$ for $i=1...m$ correspond to the input data from the training set.  The feature space in the above definition of $\phi(\mathbf{x})$ is the feature space of a Radial Basis Function (RBF) neural network   parameterised by $\delta$ with the centres corresponding to all points in the training dataset.  It is not as computationally efficient as the Gaussian kernel, but it gives a similar feature space while still allowing for the computation of $\phi(\mathbf{x})$ (which Gaussian kernel does not).  After SVM training on $\phi(\mathbf{x})$ for a given value of $\delta$, and with the soft margin parameter $C=100$, we can compute $\mathbf{w}_{s}=|\sum_{i=1}^m y_i\alpha_i\phi(\mathbf{x}_i)|$, where $\alpha_i>0$ are the support vectors found by the SVM.  With the ability to compute $\mathbf{w}_s$, we can evaluate the value of the geometric margin, $\frac{2}{|\mathbf{w}_s|}$, as well as compute the unit vector $\mathbf{w}=\frac{\mathbf{w}_s}{|\mathbf{w}_s|}$ and thus the mean margin, $\mu$, for different values of $\delta$.  

Figure \ref{eqn_margin} shows how the test error relates to the maximum geometric and mean margin values over a range of different $\delta$'s in $\phi(\mathbf{x})$ for a two-class subset problem from the smallNORB dataset \cite{LeCun.etal:2004}.  Note that the value of maximum geometric margin steadily climbs with $\delta$ despite the fact that the test error dips, reaches a minimum, and then starts climbing as the value of $\delta$ increases.  The \textit{best} performance does not correspond to the largest value of the \textit{best} geometric margin found across different $\phi(\mathbf{x})$s.  

We need to acknowledge that the lack of correlation between best test data performance and maximum value of the margin in the experiment above does not mean that there isn't an upper bound on misclassification for changing $\phi(\mathbf{x})$.  However, given that current proofs for existence of the bound require certain constraints on the structure of $\phi(\mathbf{x})$ \cite{Lanckriet.etal:2004}, the result of our simple experiment prompts us to hypothesise that in general it is the relative value of the margin within given $\phi(\mathbf{x})$ and not its absolute value across different realisations of $\phi(\mathbf{x})$ that needs to be maximised in order to improve generalisation.  This would suggest that it might not be advantageous to maximise the margin across different realisations of $\phi(\mathbf{x})$.  

Figure \ref{fig_margin} also shows the mean margin value resulting from SVM training on different realisations of $\phi(\mathbf{x})$.  Although it doesn't increase steadily with $\delta$, its maximum value does not coincide with the lowest test error either.  The value of the mean margin is not consistently related to best test performance for varying $\phi(\mathbf{x})$.

\subsection{Margin in deep architectures}

The simple experiment from the previous section suggests that maximising margin in architectures where $\phi(\mathbf{x})$ is not constant may not lead to a better generalisation.  We can go even further and show that a simple linear transformation facilitated by $\phi(\mathbf{x})$ is sufficient to produce arbitrary margin value without changing the relative position of the points with respect to the separating line given by $\mathbf{w}$.  

\begin{lemma}\label{lemma_meanmargin}
The mean margin of a set of points classified by unit vector $\mathbf{w}$, bias $\beta b$, and a feature extracting function $\phi(\mathbf{x})=\beta\widehat{\phi}(\mathbf{x})$, such that $\mu>0$, can be made arbitrarily large by varying the value of $\beta>1$.
\end{lemma}
\begin{proof}
The lemma is rather obvious, since 
\begin{equation*}
\beta\gamma_i=y_i\left [\mathbf{w}^T\left (\beta\widehat{\phi}(\mathbf{x})\right )+\beta b\right ]\mbox{,}
\end{equation*}
which produces a mean margin $\beta\mu>\mu$ if $\mu>0$ and $\beta>1$.  Note that, while bias of the linear classifier is allowed to vary, $\mathbf{w}$ remains a unit vector, as stipulated in our definition of the margin.  It is also important to note that this transformation does not change the sign of any $\gamma_i$ - all the points are classified exactly the same as before and after multiplication by $\beta$.  Thus, this transformation doesn't change anything about the classification decision in the space of $\phi(\mathbf{x})$.
\end{proof}

To understand the significance of Lemma \ref{lemma_meanmargin}, let us suppose that we are trying to maximise the mean margin in a computational model where the feature extracting function is defined by a neural network, such that 
\begin{equation*}
\phi(\mathbf{x})=f\left (\widehat{\mathbf{W}}^T\widehat{\phi}(\mathbf{x})+\widehat{\mathbf{b}}\right )\mbox{,}	
\end{equation*}
where $f(x)$ is a monotonically increasing activation function, $\widehat{\mathbf{W}}^T$ and $\widehat{\mathbf{b}}$ are the parameters of the penultimate layer, and $\widehat{\phi}(\mathbf{x})$ is the output due to all the previous layers of the network.  A simple linear transformation within $\phi(\mathbf{x})$ is sufficient to increase the margin.  The representation power of the penultimate layer is more than sufficient to provide this transformation, by changing $\widehat{\mathbf{W}}^T$ and $\widehat{\mathbf{b}}$, without any changes to $\widehat{\phi}(\mathbf{x})$ or the location of the separating hyperplane in the feature space of $\phi(\mathbf{x})$.  Thus it's possible for the data in the feature space to  stretch away from the separating hyperplane and give a larger margin without any meaningful change to the feature extraction or classification.

As a result of Lemma \ref{lemma_meanmargin} we form a hypothesis that maximisation of geometric or mean margin is not a meaningful objective for improving generalisation in deep architectures. 

\section{Margin variance}

\begin{figure}[t!]
\vskip 0.2in
\centering
\begin{subfigure}{0.6\columnwidth}
\includegraphics[width=\textwidth]{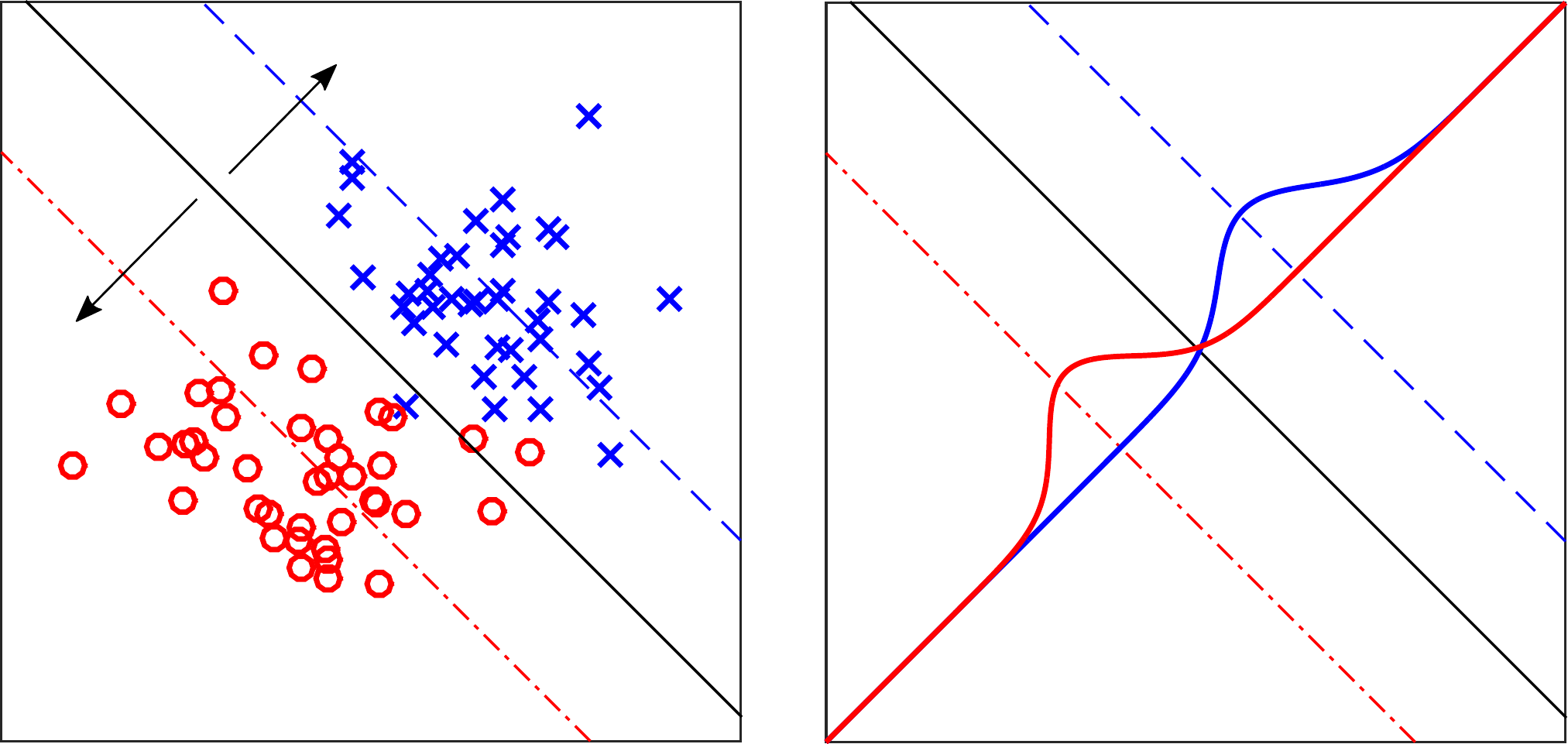}
\caption{}
\label{fig_stretch_a}
\end{subfigure}
\hskip 0.2in
\begin{subfigure}{0.6\columnwidth}
\includegraphics[width=\textwidth]{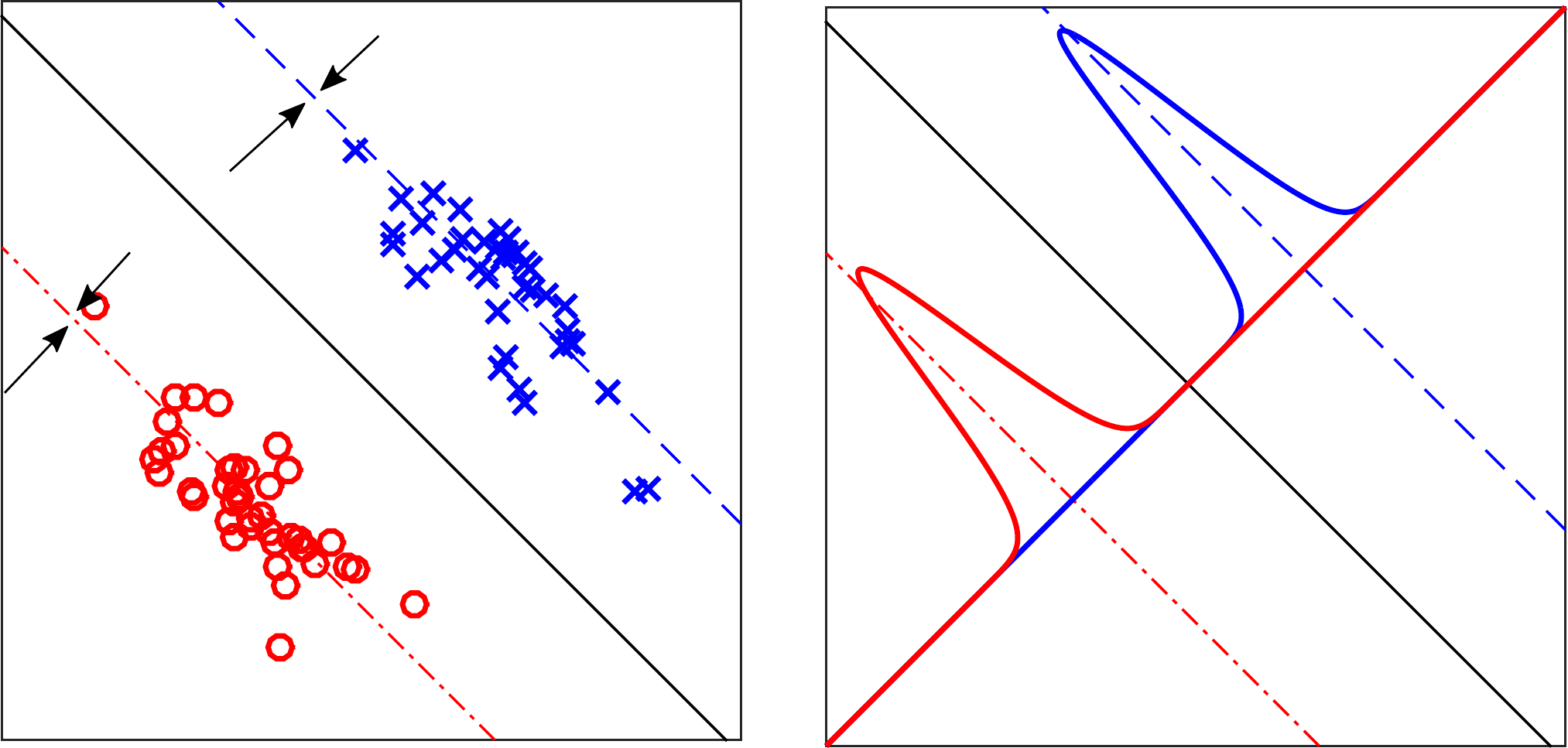}
\caption{}
\label{fig_stretch_b}
\end{subfigure}
\caption{Illustration of transformations of simple 2-dimensional $\phi(\mathbf{x})$ that increases the mean margin a); reduces the margin variance b).  The left-hand side plots shows the distribution of two-class set of points with different margin distribution and direction of the desired space transformation with the  black arrows.  The right-hand side show the distribution of the data around the margin.  The separating line is shown as a black line with the corresponding margin denoted with the red and blue dashed-lines.}  \label{fig_stretch}
\end{figure} 

Following the theory developed by \cite{Gao.etal:2013} and \cite{Zhang.etal:2016} we next consider the effect of minimising the variance of the margin in deep architectures. The variance of the margin is defined as
\begin{equation}\label{eqn_var}
\sigma=\frac{1}{m}\sum_{i=1}^m\left (\gamma_i-\mu \right)^2\mbox{.}
\end{equation}

In order to increase the mean margin, as illustrated in Figure \ref{fig_stretch_a}, it is sufficient for the feature space $\phi(\mathbf{x})$ to change so that the points stretch away from the separating hyperplane defined by $\mathbf{w}$.  This can be easily facilitated via a linear transformation, as stipulated in Lemma \ref{lemma_meanmargin}.  Figure \ref{fig_stretch_b} depicts the type of transformation that $\phi(\mathbf{x})$ needs to undergo in order to reduce the margin variance.  In addition to the stretch away from the separating hyperplane, the space must squash around two separate hyperplanes on the positive and negative margin.  It is apparent that this is a somewhat less trivial non-linear transformation, and thus more likely to be conducive to meaningful changes of $\phi(\mathbf{x})$ with respect to generalisation.  

\begin{figure}[t!]
\vskip 0.2in
\centering
\includegraphics[width=0.6\columnwidth]{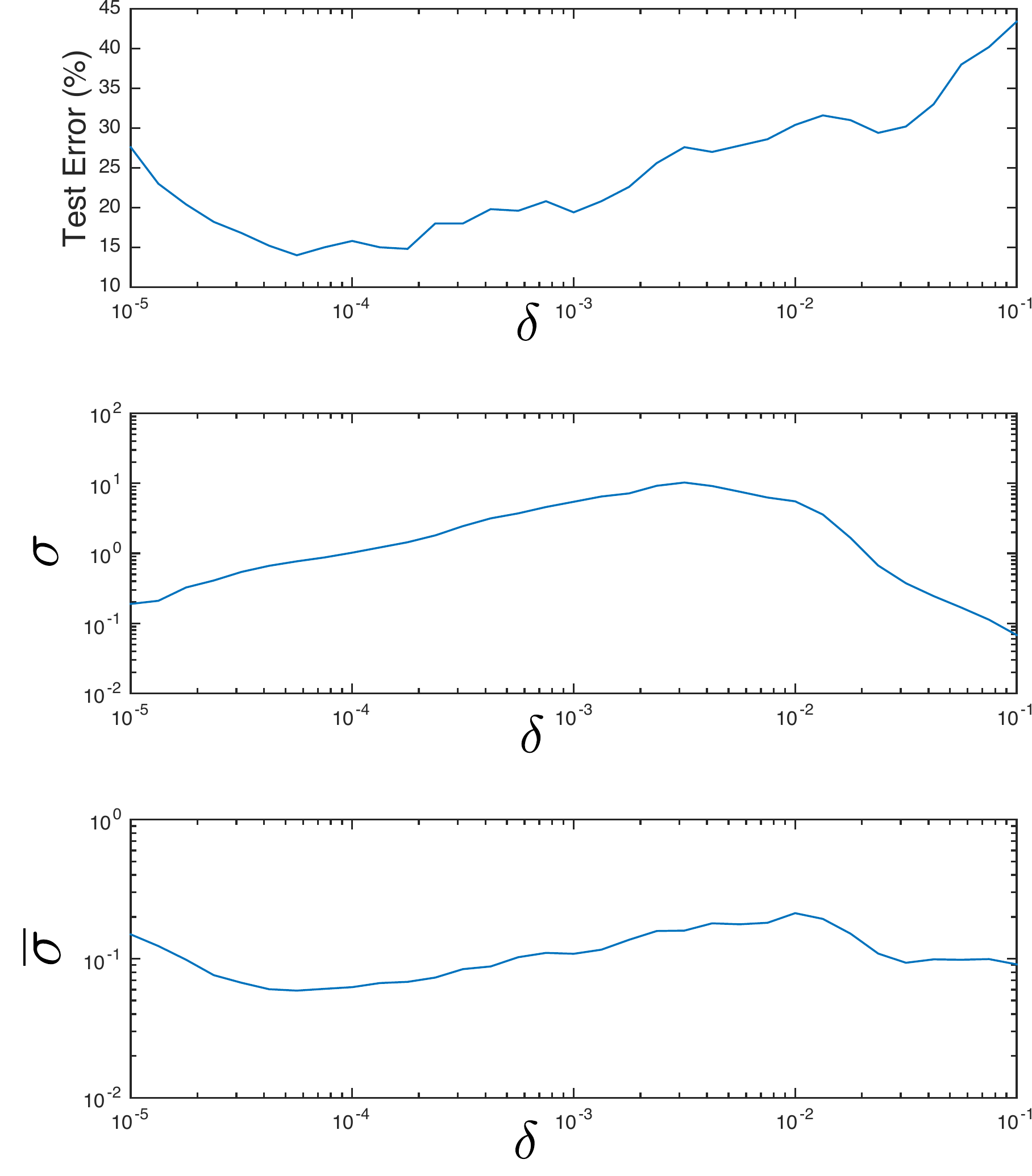}
\caption{Test error and the corresponding values for the margin variance and the normalised margin variance after SVM training on a two-class subset from the smallNORB dataset using feature extraction function defined in Equation \ref{eqn_rbfphi} over different values of parameter $\delta$. 
}\label{fig_var}
\end{figure} 

\subsection{Normalised margin}

If our hypothesis, that the mean margin value is arbitrary for changing $\phi(\mathbf{x})$, is correct it stands to reason that the variance value might be arbitrary for different $\phi(\mathbf{x})$ as well.  Indeed, if we repeat the experiment with SVM on the two-class subset of smallNORB and the feature extraction function defined in Equation \ref{eqn_rbfphi}, we can clearly see (in Figure \ref{fig_var}) that the minimum margin variance does not exactly coincide with the minimum test error in terms of the $\delta$ value that specifies the curvature of $\phi(\mathbf{x})$.  It should be noted that the SVM training does not strive to minimise the variance of the margin, but rather to maximise the geometric margin.  However, given that margin is not consistent for different $\phi(x)$, it's not unreasonable to assume that variance won't be either.  Hence we propose the normalised margin variance (NMV) defined as 

\begin{equation}\label{eqn_nmv}
\overline{\sigma}=\frac{1}{m}\sum_{i=1}^m\left (\overline{\gamma}_i-\overline{\mu} \right)^2\mbox{,}
\end{equation}
where
\begin{equation*}
\overline{\gamma}_i=\frac{\gamma_i}{\max(\{|\gamma_1|,...,|\gamma_m|\})}\mbox{,}
\end{equation*}
and
\begin{equation*}
\overline{\mu}=\frac{1}{m}\sum_{i=1}^m\overline{\gamma}_i\mbox{.}
\end{equation*}
Equation \ref{eqn_nmv} has been designed to make the margin value $-1\le\overline{\gamma}_i\le{1}$.  It is also worth to note that in a scenario where $\gamma_i=y_i\left [\mathbf{w}^T\beta\hat\phi(\mathbf{x})+\beta b\right ]$, the proposed normalisation removes the contribution of $\beta$ to the margin value.  The normalisation becomes
\begin{equation}\label{eqn_normargin2}
\overline\gamma_i=\frac{\gamma_i}{\gamma_{j}}=\frac{y_i\beta\left [\mathbf{w}^T\hat\phi(\mathbf{x})+b\right ]}{y_{j}\beta\left [\mathbf{w}^T\hat\phi(\mathbf{x}_{j})+b\right ]}\mbox{,}
\end{equation}
where $j$ is the index of the sample that produces maximum absolute value of the margin.  The $\beta$s cancel out.  This means that the linear transformation aspect of $\phi(\mathbf{x})$, which can give an arbitrary margin value at the output, is removed from the optimisation.  

Figure \ref{fig_var} shows the normalised margin variance, $\overline{\sigma}$, for different values of $\delta$ in the SVM and two-class smallNORB experiment.  The minimum of normalised margin variance does indeed fall close to the $\delta$ that gives the smallest test error.  

\subsection{Halfway loss function}

In order to carry out empirical evaluation of the effect that minimising normalised variance has on generalisation in deep architectures, we propose the Halfway loss function defined as

\begin{equation}\label{nmv_loss}
J=\frac{1}{m}\sum_{i=1}^m\left (\overline{\sigma}_i-\frac{1}{2}\right)^2\mbox{.}
\end{equation}

It is hard not to notice the resemblance of Equation \ref{nmv_loss} to the Mean Squared Error (MSE) loss function.  MSE training does in fact strive to minimise the variance of the model's output around the value given by the target label.  The point of difference between Halfway and MSE loss is the normalisation of the margin, which in effect is the same as normalisation of the model's output.

The motivation for normalisation, as discussed in the previous section, is to obtain consistency of the margin variance across different $\phi(\mathbf{x})$.  However, a consequence of this normalisation is that optimisation does not enforce an absolute target value for the output, but rather a relative value with respect to other outputs.  We hypothesise that part of the reason why Softmax is so successful in deep learning is that it allows the model to produce output in any range, as long as the relative value of the correct class neuron is larger than the value of other outputs.  This allows the deep model to operate in its \textit{natural} range, the values of the output being a result of the dynamics arising from the learner's architecture and the type of optimisation.  This natural range might be also the reason why RELU activation function  works so well with Softmax.  Normalisation of the margin assures that Halfway loss, in contrast to MSE, does allow the model to operate in its natural range, though still drives the model to produce positive and negative output in correspondence to the sign of the target label.

The Halfway loss is basically a MSE loss that minimises the margin of a classifier around half way to the current maximum value of absolute margin, $\frac{1}{2}\max(\{|\gamma_1|,...,|\gamma_m|\})$.  The choice of $\frac{1}{2}$ for the target value for the normalised margin is based on the assumption that the mean of the margin is somewhere between 0 and the current maximum value.

\subsection{Halfway loss for multi-class classification}

For multi-class classification, where label $\mathbf{y}_i=\{-1,1\}^K$ we propose a one-against-rest training scheme with a cost sensitive-learning-like \cite{Elkan:2001} multi-class weighting factor to correct the natural imbalance of the positive to negative label ratio.  In an $m$-point dataset with even distribution of $K$ classes, that is $\frac{m}{K}$ examples of each class, a given output will be trained on $\frac{m}{K}$ positive and $\frac{(K-1)m}{K}$ negative labels.  This imbalance would mean that negative labels gain more variance reduction as opposed to the positive ones.  In order to correct this, we propose the following Halfway loss for output $k$:

\begin{equation}\label{nmv_multiloss}
J_k=\frac{1}{m}\sum_{i=1}^m\overline{y}_{ki}\left (\overline{\gamma}_{ki}-\frac{1}{2}\right)^2\mbox{,}
\end{equation}
where
\begin{equation*}
\overline{y}_{ki}=\begin{cases}1 & y_{ki}=1 \\ \frac{1}{K-1} & y_{ki}=-1\mbox{.}\end{cases}
\end{equation*}
The symbols $\overline{\gamma}_{ki}$ and $y_{ki}$ represent the normalised margin and the target label of the $k^{\mbox{th}}$ output for input $i$.  The multi-class weighting factor $\overline{y}_{ki}$ can be derived from the label as follows:
\begin{equation*}
\overline{y}_i = \frac{(K-2)y_i+K}{2K-2}\mbox{.}
\end{equation*}
Note that the two-class Halfway loss defined in Equation \ref{nmv_loss} is analogous to $K=2$ case of the multi-class Halfway loss defined in equation \ref{nmv_multiloss}.  

\section{Empirical evaluation of Halfway loss}

\begin{table}[t!]
\begin{center}
\caption {Averate test error (in \% of misclassified samples) for the Softmax vs Halfway training over 10 trials on the MNIST, smallNORB and CIFAR-10 datasets and three deep architectures. 
}\label{tbl_multiclass} 
\begin{tabular}{| l |c | c |}
\hline
 & Softmax & Halfway \\
 \hline
\multicolumn{3}{| l | }{MNIST} \\
\hline
FC-128-32 		& 2.26 $\pm$ 0.09 & 2.32 $\pm$ 0.11 \\
FC-500-500-2000 & 1.96 $\pm$ 0.11 & 1.62 $\pm$ 0.12 \\
CNN 			& 0.62 $\pm$ 0.05 & 0.55 $\pm$ 0.04 \\
\hline
\multicolumn{3}{| l | }{smallNORB} \\
\hline
FC-128-32 		& 31.99 $\pm$ 2.75 & 31.75 $\pm$ 5.56 \\
FC-500-500-2000 & 27.84 $\pm$ 3.55 & 22.64 $\pm$ 1.50 \\
CNN 			& 12.51 $\pm$ 1.07 & 11.01 $\pm$ 1.16 \\
\hline
\multicolumn{3}{| l | }{CIFAR10} \\
\hline
FC-128-32 		& 50.12 $\pm$ 0.50 & 51.65 $\pm$ 0.69 \\
FC-500-500-2000 & 51.29 $\pm$ 0.49 & 46.60 $\pm$ 0.43 \\
CNN 			& 30.06 $\pm$ 0.94 & 27.83 $\pm$ 0.47 \\
\hline 
\end{tabular}
\end{center}
\end{table}
%\brenNote{Table \ref{tbl_multiclass}: A simple t-test for significance would be quite useful here.}

\begin{figure}[t!]
\vskip 0.2in
\centering
\begin{subfigure}{0.6\columnwidth}
\includegraphics[width=\textwidth]{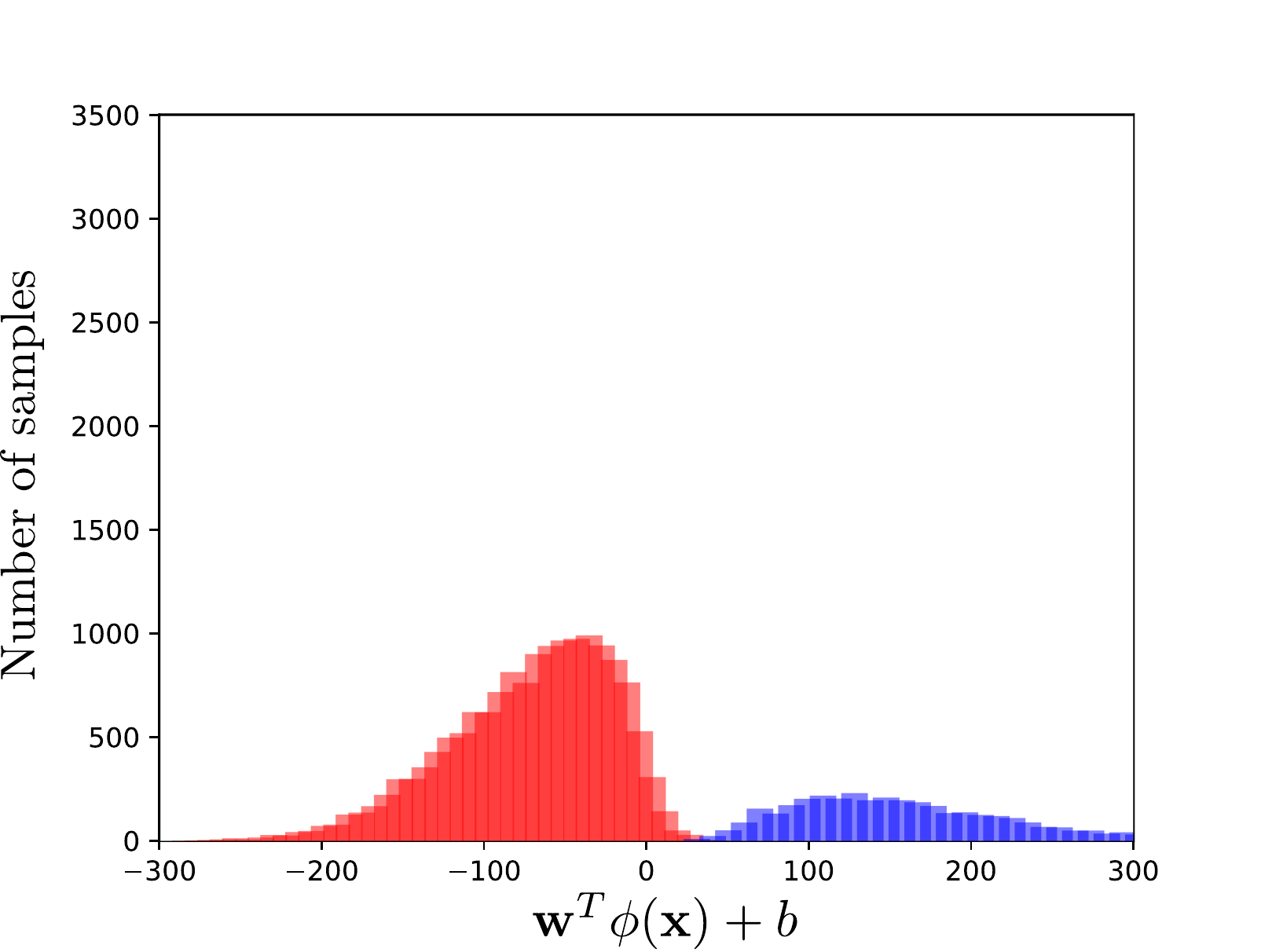}
\caption{Softmax loss}
\label{fig_dist_smax}
\end{subfigure}
\hskip 0.2in
\begin{subfigure}{0.6\columnwidth}
\includegraphics[width=\textwidth]{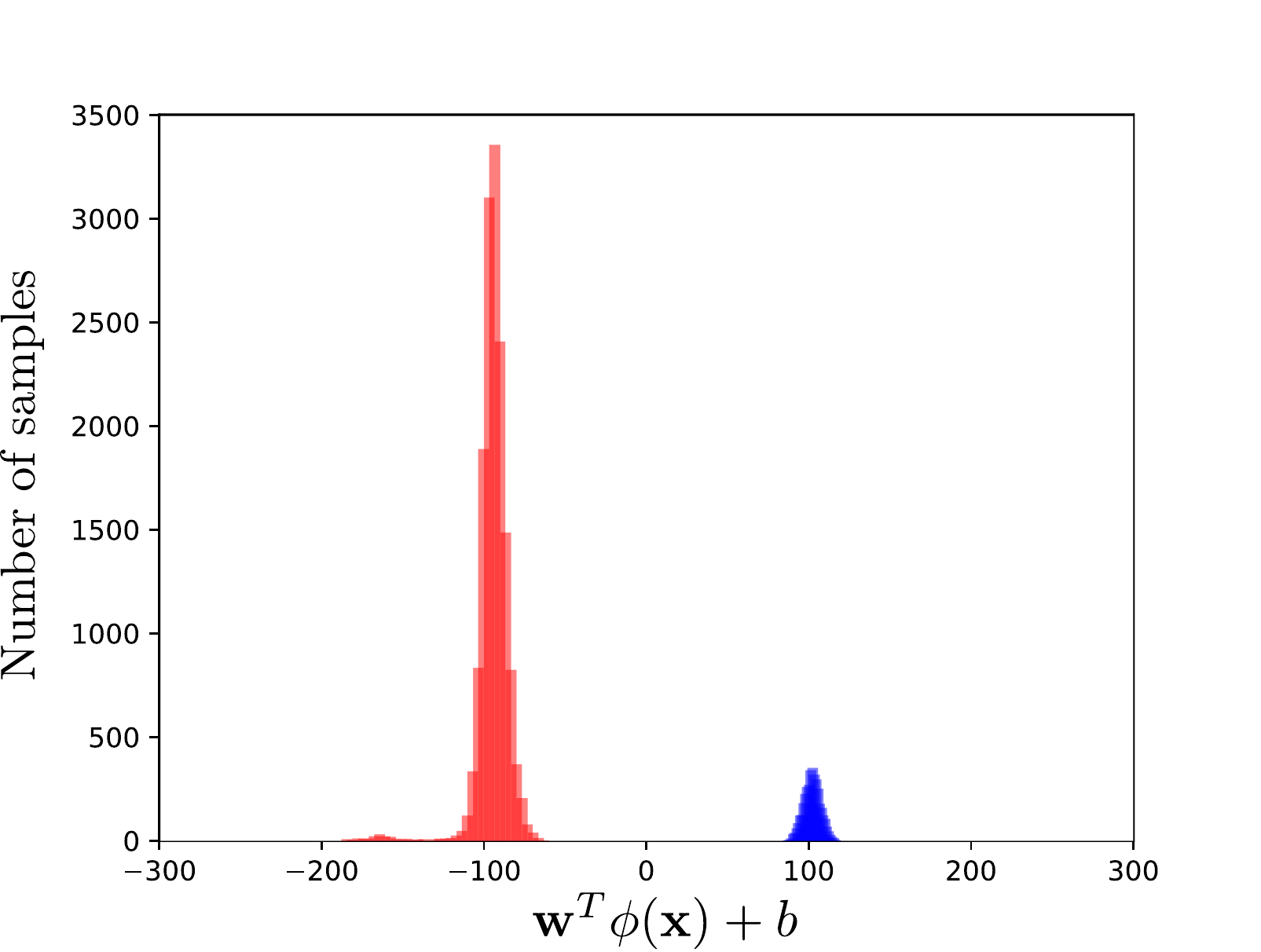}
\caption{NMV loss}
\label{fig_dist_nmv}
\end{subfigure}
\caption{Distribution of the positively (blue) and negatively (red) labelled output of the train data of a single output of a smallNORB-trained CNN.    }\label{fig_dist}
\end{figure} 

\begin{figure}[t!]
\vskip 0.2in
\centering
\includegraphics[width=0.55\columnwidth]{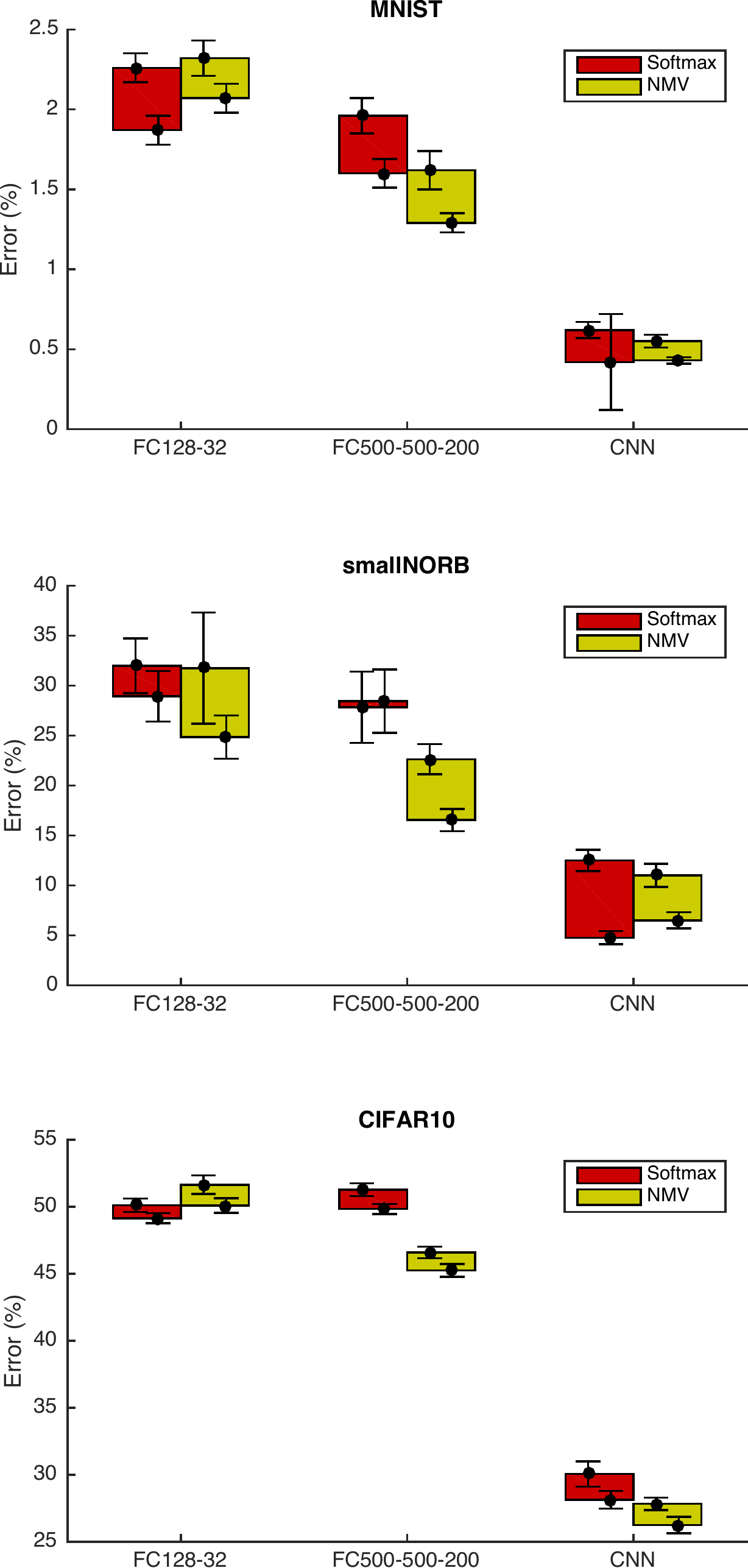}
\caption{Generalisation error between the validation and test data represented as length of the bars.  The top of the bar is placed at the test error mean and bottom at the validation error mean for 10 trials of different architectures and datasets.  The standard deviation of the test error is shown on the left hand-side of each bar, the standard deviation for the validation error shown on the right-hand side of each bar.  The errors for Softmax are shown in red and for Halfway in yellow.}
\label{fig_valid}
%\hskip 0.2in
\end{figure} 

The three datasets that we will use for evaluation of the Halfway loss are the MNIST \cite{LeCun.etal:1998}, smallNORB \cite{LeCun.etal:2004} and CIFAR-10 \cite{Krizhevsky:2009} datasets.  Each set comes pre-divided into a training (60000 MNIST, 24300 smallNORB, 50000 CIFAR-10) and testing (10000 MNIST, 243000 smallNORB, 10000 CIFAR-10) set of samples.  For our evaluation, we have split each training part of the dataset into a set of images used to train the models (55000 MNIST, 19440 smallNORB, 45000 CIFAR-10), and a validation set (5000 MNIST, 4860 smallNORB, 5000 CIFAR-10) used to determine the best manifestation of the model.  The choice of the validation sample for MNIST and CIFAR-10 was made randomly, whereas for the smallNORB dataset it was all the images of a specific instance of each of the class of toys, with a random choice of which instance was used for validation.  For the smallNORB dataset only the left camera images were used.  

	Training was done with mini-batch optimisation, with a 500 sample batch-size for MNIST and CIFAR-10, and a 405 sample batch-size for smallNORB.  The normalisation of the margin was carried out independently in each batch, which makes the Halfway loss somewhat similar to the batch normalisation transform proposed by \citet{Ioffe.etal:2015}.  However, whereas the objective of batch normalisation requires computing the mean and variance across the batch sample in order to normalise its first and second order statistics, our normalisation of the margin only divides the data by the maximum absolute value of the output.  Also, the aim of producing unity variance of the batch sample in batch normalisation, regardless of the label, is counter-objective to ours, which is to minimise the variance around the margin of different labels.  Nevertheless, it is quite possible that Halfway loss minimisation with min-batch training shares some of the effects of reducing the \textit{internal covariate shift} of batch normalisation in the output layer of the network. 
	
All the evaluations were done using TensorFlow \cite{tensorflow2015-whitepaper}, which provides automatic computation of the gradients required for optimisation.  In our implementation there was no constraint placed on $\mathbf{w}$ to make it a unit vector, since margin normalisation, as shown in Equation \ref{eqn_normargin2}, has the same effect making individual normalisation of $\mathbf{w}$ irrelevant.

Three different models were used for classification of each dataset.  The small fully connected neural network (FC-128-32) consisted of 2-hidden layers with 128 and 32 neurons in the consecutive layers.  The big fully connected neural network (FC-500-500-2000) consisted of 3-hidden layers with 500, 500 and 2000 neurons in the consecutive layers. Finally a CNN model was used for classification in each dataset.  For the MNIST and smallNORB dataset a CNN consisted of two convolutional layers.  The first convolutional layer had 32 filters of 5x5 size and stride of 1 followed by a 2x2 input max-pooling with stride 2; the second convolutional layer had 64 filters of 5x5  size and stride 1 followed by a 2x2 input max-pooling with stride of 2; this was followed with a 512-neuron fully connected layer and 0.5 probability dropout during training.  For the CIFAR-10 model the CNN consisted of two convolutional layers also.  The first convolutional layer had 54 filters of 5x5 size and stride 1 followed by a 3x3 input max-pooling with stride of 2 and local response normalisation; the second convolutional layer consisted of 64 filters of 5x5 size and stride of 1 followed by local response normalisation and a 3x3 input max-pool layer with stride of 2; this was followed by two fully connected layers of 384 and 192 neurons.  The activation function used in all networks was ReLU.

The optimisation for all tests was done using Tensorflow's implementation of the Adam optimiser \cite{Diederik.etal:2014}.  The learning rate for all runs was set to $0.001$ and the maximum number of training epochs, with one epoch training over the entire set of mini-batch blocks, was set to 2000.  There was no regularisation in any of the models, as the purpose of this evaluation was to compare the sole effect of the compared loss functions.  Hence, rather than going for the state of the art results, we are aiming for a comparative study of the effect of minimisation of margin variance as compared to Softmax with Cross-Entropy loss function training.  The $\mathbf{y}\in\{0,1\}^K$ coding was used for target labels during Softmax training and $\mathbf{y}\in\{-1,1\}^K$ coding was used for target labels during Halway loss training.

Table \ref{tbl_multiclass} reports the average test error and variance over 10 trials with the same initial values of weights and biases for a given trial between the Softmax and Halfway optimisation.  Test error was measured by taking the output of with the maximum value to indicate the index of the identified class.  The reported test error comes from the model state at the training epoch that produced the lowest validation error.  The Halfway loss consistently leads to lower test error, sometimes by quite a significant amount.

Figure \ref{fig_dist} shows the distribution of the values across a single output, $\mathbf{w}\phi(\mathbf{x})+b$, from the entire train set on the smallNORB-trained CNN using Softmax and Halfway optimisation.  Halfway trained model does indeed produce output with a smaller variance around the margins while maintaining the values in a similar range to the Softmax trained model.

It is also interesting to examine the difference between the validation and test error.  In some way, it gives an idea of the generalisation error.  Validation error stands for the empirical risk,  since it was used during training to choose the best model (deemed to be the one that gives minimum validation error).  The test error, although still just an average, simulates the true risk, since it has not been seen by the learner during the training.  Figure \ref{fig_valid} shows a bar plot comparing the generalisation error between Softmax and Halfway training for tested datasets and architectures.  The length of the bars in the plot corresponds to the generalisation error; top and the bottom positions of each bar demarks the test and validation error respectively.  The desired characteristic is for the top of the bar to be as lows as possible (low test error) and the bar to be as short as possible (validation error being close to the test error).  Although not in every single case, Halfway looks to outperform Softmax in a combination of lower test error and/or smaller generalisation error.

\section{Discussion}

When it comes to the bigger models, FC-500-500-2000 and CNN, those trained with the Halfway loss consistently outperform those trained with the Softmax Cross-Entropy in terms of the mean and also the standard deviation of the test error over multiple trials of different initial conditions.  At the same time, for the small network, FC-128-32, Halfway training performs consistently worse (although, aside from smallNORB, it is only a bit worse).  An intuitive explanation for this is that Halfway loss is more constrained than Softmax in terms of what it demands of distribution of points in $\phi(\mathbf{x})$.  While these constraints are demonstrably favourable to generalisation in representationally rich models, they might be getting in the way of class separation objective in representationally limited models.  In other words, Halfway loss may provide a better objective for classification, but an objective that is a bit harder to attain in models with limited transformation dynamics.  

We also found that the cost-sensitive learning aspect of the Halfway loss was critical for its good performance.  This is most interesting given that previously work by \citet{Zhou.etal:2006} found cost-sensitive learning not to be useful for multi-class one-against-rest optimisation, albeit for different loss functions.  The motivation for class weighting in Halfway loss is to ensure that the optimisation does not drive the variance around the negative margin to a smaller value than the variance around the positive margin.  We take the need for class balancing in Halfway training as a confirmation of our assumption that an even reduction of variance around the positive and negative margin is critical for good generalisation.

The MSE-like nature of the Halfway loss has a disadvantage in that it presumes a normal distribution of the data around the margin.  While it does succeed in minimising the margin variance, it also produces a symmetric distribution of data around the margin (as Figure \ref{fig_dist_nmv} shows).  It is possible that minimisation of variance while producing distributions skewed away from the margin might improve the generalisation even further.    

\section{Conclusion}

We have taken the ideas around margin distribution from boosting theory and applied them to deep learning.  The driving hypothesis of our work was that maximisation of the margin alone is not a useful objective for architectures where the feature extraction function changes during optimisation.  However, minimisation of margin variance might be.  We have provided some theoretical evidence that maximisation of margin in a neural network might be trivial.  This we followed with empirical investigation of the importance of margin variance.

We proposed the Halfway loss function as the training objective that minimises the normalised margin variance .  It's an MSE-like training objective with cost-sensitive learning that aims to reduce variance around halfway point between 0 and maximum margin value (as calculated from the training dataset).  Our empirical evaluation on known image datasets demonstrates superiority of Halfway over the Softmax Cross-Entropy loss in representationally rich fully connected, as well as convolutional, neural networks.  We also confirmed that in the balance of things, Halfway loss does seem to provide better generalisation - in terms of producing a validation test score that is a better estimation of the test score, while ensuring better test data performance.
 
For the future work, given the empirical evidence this work presents, we believe it would be worthwhile to find theoretical proofs that establish the significance of margin variance as well as the irrelevance of the margin mean for generalisation in deep architectures.  On the empirical side, it might be also possible to form better loss functions which minimise margin variance but do not enforce symmetric distribution of the points around the margin, and thus possibly lead to even better generalisation. 

\section*{Acknowledgements} 

We gratefully acknowledge the support of NVIDIA Corporation with the donation of the Titan X GPU used for this research.
 

\bibliography{references}
\bibliographystyle{icml2017}

\end{document}